\documentclass[10pt,twocolumn,letterpaper]{article}

\usepackage[pagenumbers]{cvpr} 

\usepackage{graphicx}
\usepackage{amsmath}
\usepackage{amssymb}
\usepackage{booktabs}

\usepackage[pagebackref,breaklinks,colorlinks]{hyperref}

\usepackage{lineno,hyperref}
\usepackage[final]{changes}
\usepackage{amsmath}
\usepackage{amsthm}
\usepackage{mathtools}
\usepackage{threeparttable}
\usepackage{booktabs}
\usepackage{makecell}
\usepackage{multirow}
\usepackage{bbding}
\usepackage{algorithm}
\usepackage{algpseudocode}  
\usepackage{newfloat}
\usepackage{listings}

\newtheorem{theorem}{Theorem}
\newtheorem{assumption}{Assumption}
\newtheorem{heuristic}{Heuristic}

\usepackage[capitalize]{cleveref}
\crefname{section}{Sec.}{Secs.}
\Crefname{section}{Section}{Sections}
\Crefname{table}{Table}{Tables}
\crefname{table}{Tab.}{Tabs.}

\def\cvprPaperID{*****} 
\def\confName{CVPR}
\def\confYear{2023}

\makeatletter
\renewcommand\paragraph{\@startsection{paragraph}{4}{\z@}%
  {-0.1\baselineskip}   
  { 0.0\baselineskip}   
  {\normalfont\normalsize\bfseries}}
\makeatother

\begin{document}

\title{Incorporating brain-inspired mechanisms for multimodal learning
\\in artificial intelligence}

\author{
Xiang He\textsuperscript{1}\thanks{Equal Contribution}, 
Dongcheng Zhao\textsuperscript{1,2}\footnotemark[1],
Yang Li\textsuperscript{1},  
Qingqun Kong\textsuperscript{1}\thanks{Corresponding Author}, 
Xin Yang\textsuperscript{3}\footnotemark[2],  
Yi Zeng\textsuperscript{1,2,4}\footnotemark[2]\\
{\small \textsuperscript{1}Brain-inspired Cognitive AI Lab, Institute of Automation, Chinese Academy of Sciences, Beijing, China} \\
{\small \textsuperscript{2}Center for Long-term Al, Beijing, China} \\
{\small \textsuperscript{3}CAS Key Laboratory of Molecular Imaging, Institute of Automation, Chinese Academy of Sciences, Beijing, China} \\
{\small \textsuperscript{4}Key Laboratory of Brain Cognition and Brain-inspired Intelligence Technology, Chinese Academy of Sciences, Shanghai, China} \\
{\tt \small {\{hexiang2021, zhaodongcheng2016, liyang2019, qingqun.kong, xin.yang, yi.zeng\}}@ia.ac.cn}
}

\maketitle


\begin{abstract}
  Multimodal learning significantly enhances the perceptual capabilities of cognitive intelligent systems by integrating information from different sensory modalities. However, existing multimodal fusion researches in the field of artificial intelligence typically assume static integration of modal information, not yet fully incorporating the key dynamic mechanisms of multimodal integration found in the brain. Specifically, when processing multisensory information, the brain exhibits an inverse effectiveness phenomenon, wherein weaker unimodal cues yield stronger multisensory integration benefits; conversely, when individual modal cues are stronger, the effect of modal fusion is relatively diminished. This mechanism enables biological systems to achieve robust cognition even in environments with scarce or noisy perceptual cues. 
  Inspired by this biological inverse effectiveness mechanism in multimodal integration, we explore the intrinsic relationship between multimodal output and information from individual modalities, proposing an inverse effectiveness driven multimodal fusion (IEMF) strategy. By incorporating this inverse effectiveness-driven multimodal fusion strategy into neural network architectures, we achieve not only more efficient multimodal integration with significantly improved model performance, but also substantial computational efficiency gains—demonstrating up to 50\% reduction in computational cost across diverse fusion methods.
  We conduct extensive experiments on audio-visual classification, audio-visual continual learning, and audio-visual question answering tasks to validate the effectiveness of our proposed method. The experimental results consistently demonstrate that our proposed method performs excellently in these multimodal tasks. Furthermore, to verify the universality and generalization capability of our method, we also conduct experiments on two widely used network models in artificial intelligence—Artificial Neural Networks (ANN) and Spiking Neural Networks (SNN)—with results showing good adaptability of the method to both network types.
Our research emphasizes the potential of incorporating biologically inspired neural mechanisms into multimodal neural networks and provides promising new directions and perspectives for the future research and development of multimodal artificial intelligence.
The code is publicly available at \texttt{https://github.com/Brain-Cog-Lab/IEMF}.
\end{abstract}

\begin{figure*}[t]
	\centering
		\includegraphics[width=1.0\linewidth]{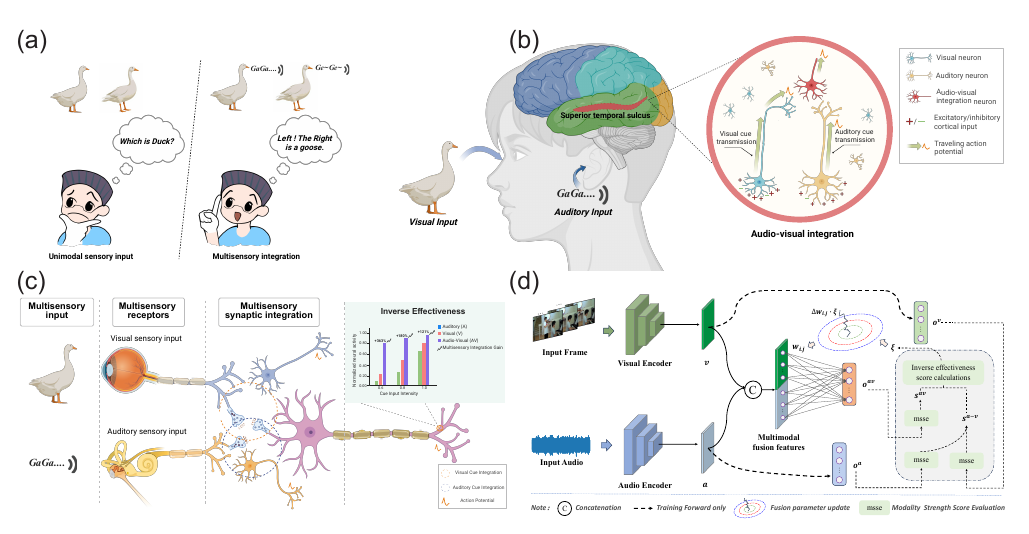}
	\caption{\textbf{Illustration of multisensory integration and the role of inverse effectiveness in IEMF (Inverse Effectiveness driven Multimodal Fusion).} 
  \textbf{(a)} Comparison between unimodal sensory input and multisensory integration: integrating visual and auditory cues reduces ambiguity and uncertainty compared to relying on a single modality.  
  \textbf{(b)} Neural basis of audiovisual integration in the human brain, focusing on the superior temporal sulcus (STS) where visual and auditory inputs converge onto multisensory neurons.
  \textbf{(c)} Biological principle of inverse effectiveness: multisensory integration is strengthened when unimodal signals are weak. Visual and auditory stimuli are processed through distinct sensory pathways and converge at multisensory synapses. The inset illustrates the inverse relationship between unimodal strength and integrative gain. 
  \textbf{(d)} The proposed inverse effectiveness driven multimodal fusion strategy inspired by biological multisensory fusion mechanisms. Visual and auditory inputs are processed by respective encoders, fused via a dynamic fusion module regulated by inverse effectiveness principles, and evaluated using modality strength score estimation. The fusion module weights are dynamically adjusted according to the computed scores.
  (This figure was created with \url{https://BioRender.com}.)}
	\label{fig1}
\end{figure*}

\section{Introduction}
\label{sec:intro}


In natural environments, we typically need to process cues from multiple senses simultaneously to comprehensively construct an understanding of the same concept. For example, understanding the concept of "beach" involves not only visual information (yellow sand, blue sea) but also relies on auditory (sound of waves) and tactile (texture of sand) sensory information. Compared to unimodal information, multimodal information provides richer and more comprehensive representational capacity~\cite{ernst2004merging, noppeney2021perceptual}. As shown in Fig.~\ref{fig1}(a), multimodal integration not only enhances information expressiveness but also effectively reduces uncertainty in single-modal information. This mechanism of multimodal information integration is not only the foundation of biological perception but has also become one of the core challenges in multimodal learning in artificial intelligence.
As information environments become increasingly complex, traditional unimodal learning methods struggle to handle complex and dynamic real-world scenarios. Consequently, neural networks have incorporated multimodal information processing strategies to achieve more robust and efficient information representation. Multimodal neural networks are widely applied in tasks such as multimodal fusion~\cite{nagrani2021attention, peng2022balanced, yu2023brain, jiang2023mammalian, sadaf2023bio}, multimodal emotion recognition~\cite{lv2021progressive, cheng2024emotion}, and audio-visual speech recognition~\cite{mroueh2015deep, kim2021cromm, peng2022balanced, yeo2024akvsr}.
Nevertheless, brain-inspired algorithms remain in the developmental stage, and many biological characteristics and mechanisms have not yet been fully utilized, offering enormous potential and new challenges for the further development of neural network models.


Neurobiological research indicates that vision and audition are the two primary pathways through which humans acquire external information, and their integration significantly enhances perceptual benefits~\cite{bulkin2006seeing, enoch2019evaluating}. This paper therefore focuses primarily on the integration of visual and auditory modal inputs. For visual and auditory information from a common source, the brain has specialized regions responsible for both unimodal information processing and multisensory integration~\cite{macaluso2005multisensory}. After visual and auditory information are received through their respective receptors, features are hierarchically extracted through visual and auditory pathways before being transmitted to multisensory integration brain regions.
Relevant studies show that audiovisual information integration occurring in the cerebral cortex is closely associated with regions such as the superior temporal sulcus~\cite{calvert2000evidence, macaluso2004spatial, noesselt2007audiovisual, van2004integration, senkowski2008look, szycik2008novel}, posterior parietal cortex~\cite{avillac2005reference, regenbogen2018intraparietal}, and prefrontal cortex~\cite{bushara2001neural, barraclough2005integration}. Figure~\ref{fig1}(b) illustrates the convergence and integration process of co-sourced audiovisual signals in the superior temporal sulcus as an example. In this multisensory integration brain region, visual and auditory cues are transmitted through different neural pathways, ultimately converging onto common multisensory integration neurons, facilitating cross-modal information integration and perceptual decision-making.

Brain audiovisual information integration exhibits many interesting mechanisms, with inverse effectiveness being particularly noteworthy. \cite{calvert2000evidence} found that by separately presenting audiovisual combined speech signals and their unimodal information, and conducting cross-modal comparisons, the left superior temporal sulcus (STS) demonstrated the most significant cross-modal integration benefits under conditions where unimodal signals were weakest. The inverse effectiveness mechanism indicates that during multisensory information integration, when unimodal cues are weaker, the effect of multisensory integration is relatively stronger; conversely, when individual modal cues are stronger, the effect of modal fusion is relatively diminished, though multisensory integration responses still exceed the activation response of either single modality \cite{stein2008multisensory, fetsch2013bridging}. Figure~\ref{fig1}(c) illustrates this phenomenon, where inverse effectiveness reflects higher sensitivity to weaker modalities in multimodal integration brain regions, typically manifested as enhanced information integration. This mechanism enables biological systems to enhance perceptual accuracy and stability by strengthening multimodal integration when the quality of information from a single modality is poor.


Inspired by the biological principle of inverse effectiveness, our work reconsiders how multimodal fusion should adapt to variations in unimodal input quality, particularly under dynamic and complex environments. Most existing multimodal fusion methods focus on maximizing information interaction between modalities, yet often overlook the dynamic relationship between the fused output and the respective contributions of each modality. This limitation stems from existing methods typically presetting modal interactions as static fixed patterns, failing to fully consider that different modalities' information contributions should flexibly adjust as environmental conditions change.
Take audiovisual perception as an example: when environmental noise significantly degrades the quality of auditory input, traditional fusion strategies typically retain fixed fusion weights and cannot adaptively modify the cooperative interaction between modalities in accordance with signal degradation, thereby constraining the overall perceptual performance of the system. The phenomenon of inverse effectiveness inspires the insight that an efficient multimodal integration mechanism should actively enhance the responsiveness of the fusion module when the quality of a single modality degrades, so that the system can obtain more compensatory information from other modalities. Based on this insight, we propose that fusion strength should dynamically respond to modality-specific quality fluctuations, that is, the learning rate of the fusion module should be adaptively modulated according to the reliability of unimodal signals, thereby enabling more robust and flexible multimodal perception in complex and evolving environments.

Based on the aforementioned neural mechanisms of multimodal fusion and biological inspiration, this paper adopts deep neural networks as the foundational framework for multimodal perceptual learning, focusing on exploring the cooperative integration process of visual and auditory information.
We propose an inverse effectiveness driven multimodal fusion (IEMF) strategy to enable a more fine-grained fusion mechanism. By quantifying the relationship between the strength of unimodal inputs and the signal strength of multimodal fusion outputs, we adaptively modulate the update rate of the fusion module’s weights. Specifically, we introduce an inverse effectiveness coefficient into the backpropagation process, such that the fusion module accelerates its parameter updates in response to weak unimodal signals to enhance fusion strength, while suppressing updates when unimodal signals are strong, thereby reducing over-reliance on fusion. This design realizes a biologically inspired principle of ``weak modality, strong fusion,'' ensuring that the integration process neither over-depends on a single sensory pathway nor overlooks potentially informative sources. Consequently, the method effectively improves overall perceptual accuracy and model robustness.
Beyond improving overall perceptual accuracy and model robustness, our approach also demonstrates significant computational efficiency gains—reducing training costs by up to 50\% while maintaining superior performance. These dual benefits highlight the computational advantages of inverse effectiveness principles in multimodal integration systems.

Overall, our contributions can be categorized into the following three points:
\begin{itemize}
\item We introduce the inverse effectiveness mechanism into multimodal fusion in deep neural networks for the first time, proposing an inverse effectiveness driven multimodal fusion strategy. This strategy adjusts the parameter update intensity of fusion modules in real-time, enabling the model to enhance its ability to extract information from other modalities when a single modality signal is weak, thereby improving information compensation effects and fusion efficiency.
\item We evaluate our proposed method on two different architectures: Artificial Neural Networks (ANN) and Spiking Neural Networks (SNN). Experimental results demonstrate that IEMF possesses good generality and can be effectively integrated with both types of networks, fully leveraging its mechanism advantages.
\item We conduct systematic empirical studies on multiple standard datasets and tasks, including representative scenarios such as audiovisual speech recognition, audiovisual continual learning, and audiovisual question answering. Experimental results show that our proposed method exhibits stronger perceptual capabilities under various complex conditions. Particularly worth emphasizing is that, as a mechanism, IEMF can seamlessly integrate with various existing state-of-the-art methods and further enhance their performance.
\end{itemize}

\begin{figure*}[t]
	\centering
		\includegraphics[width=0.9\linewidth]{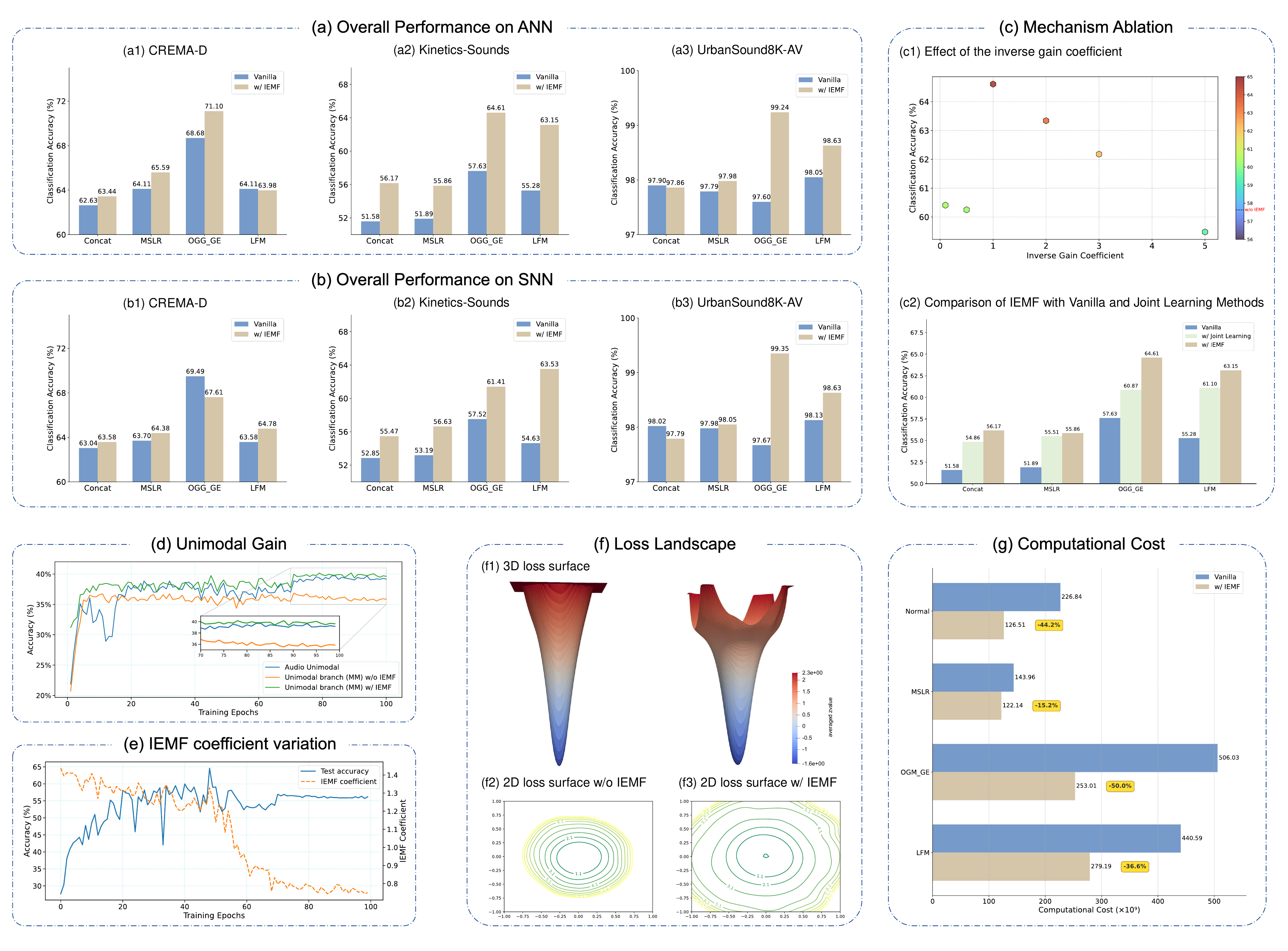}
    \caption{\textbf{Comprehensive evaluation of the proposed inverse effectiveness driven multimodal fusion (IEMF).}
    \textbf{(a)} Overall performance on ANNs.  
    Bar charts compare the vanilla method (\textcolor{blue}{blue}) with the method augmented by IEMF (\textcolor{brown}{khaki}) on three audiovisual classification benchmarks—CREMA-D \textbf{(a1)}, Kinetics-Sounds \textbf{(a2)} and UrbanSound8K-AV \textbf{(a3)}—under four representative fusion schemes (Concat, MSLR, OGM\_GE and LFM).  
    \textbf{(b)} Overall performance on SNNs.  
    Same layout as (a) but using spiking neural networks, demonstrating that IEMF consistently boosts accuracy across network paradigms and datasets \textbf{(b1–b3)}.  
    \textbf{(c)} Mechanism ablation on the Kinetics-Sounds dataset.
    \textbf{(c1)} Effect of the inverse gain coefficient $\gamma$: the baseline without IEMF (grey dashed line on the colour bar) scores below every IEMF setting; model accuracy peaks at $\gamma{=}1$.
    \textbf{(c2)} Removing the IEMF term (``Joint Learning only'') leads to a clear performance drop, highlighting the essential role of the inverse effectiveness multimodal fusion component.
    \textbf{(d)} Unimodal gain analysis.  
    Test accuracy for (i) a unimodal audio model trained alone (``Unimodal''), (ii) the audio branch extracted from a multimodal model without IEMF (``Unimodal branch (MM) w/o IEMF''), and (iii) the audio branch with IEMF (``Unimodal branch (MM) w/ IEMF''). IEMF yields a persistent relative gain for the unimodal branch.  
    \textbf{(e)} Dynamics of the IEMF coefficient.  
    Evolution of the learnt IMEF coefficient $\xi$ (dashed orange, right {\small y}-axis) alongside the test accuracy (solid blue, left {\small y}-axis) during training. At the early stage, the value of $\xi$ is large to accelerate the multimodal integration, while as the network converges, $\xi$ falls back and remains stable to maintain the fusion stability.
    \textbf{(f)} Loss landscape visualization.
\textbf{(f1)} 3D loss surface comparison: vanilla method (left) versus IEMF-enhanced method (right).
\textbf{(f2-f3)} 2D contour plots: without IEMF (f2) versus with IEMF (f3). The IEMF method leads to broader and flatter minima.
\textbf{(g)} Computational cost analysis.
Comparison of computational cost between standard models (\textcolor{blue}{blue}) and IEMF-enhanced models (\textcolor{brown}{khaki}). IEMF significantly reduces computational costs across all fusion methods, with reductions ranging from 15.2\% to 50.0\% (highlighted in yellow percentages).}
	\label{fig2}
\end{figure*}

\section{Results}
\noindent
\textbf{IEMF's generalizability across network architectures}\par
\noindent
A key advantage of IEMF lies in its strong universality across different neural network architectures. To evaluate this characteristic, we integrated IEMF into both Artificial Neural Networks (ANN) and Spiking Neural Networks (SNN), which represent distinctly different information processing paradigms. As shown in Fig.~\ref{fig2}(a) and~\ref{fig2}(b), IEMF consistently improves performance across multiple audio-visual classification benchmark tasks, regardless of the underlying network type. For example, with ANNs, IEMF increased classification accuracy from 51.58\% to 56.17\% under the Concat fusion on the Kinetics-Sounds dataset; similarly with SNNs, IEMF brought stable performance gains, improving model accuracy from 52.85\% to 55.47\%, verifying the robustness and flexibility of our proposed method.

This cross-architecture robustness is particularly important for practical applications, as real-world systems often employ heterogeneous network models due to hardware resource constraints, power limitations, or real-time processing requirements. Notably, IEMF achieved significant benefits even in scenarios where traditional multimodal fusion methods are limited by the sparsity and event-driven characteristics of spiking neural networks. For instance, when using the LFM fusion method on the Kinetics-Sounds dataset, the original SNN accuracy was 54.63\%, slightly lower than the ANN's 55.28\%; however, after introducing IEMF, the SNN classification accuracy surpassed the ANN, reaching 63.53\% compared to the ANN's 63.15\% with IEMF, as seen in Fig.~\ref{fig2}(b2) and Fig.~\ref{fig2}(a2). These results indicate that IEMF is not limited to traditional architectures but provides a universally applicable mechanism for improving multimodal fusion across various neural network architectures.

\bigskip
\noindent
\textbf{IEMF improved the model performance on audio visual classification task}\par
\noindent
We systematically validated the effectiveness of the inverse effectiveness driven multimodal fusion (IEMF) mechanism in audio-visual classification tasks. We evaluated performance differences between baseline models and IEMF-enhanced models across three representative datasets: CREMA-D~\cite{cao2014crema}, Kinetics-Sounds~\cite{arandjelovic2017look}, and UrbanSound8K-AV~\cite{guo2023transformer}, using four mainstream fusion strategies: Concatenation Fusion (Concat), Modality-Specific Learning Rates (MSLR)~\cite{yao2022modality}, On-the-fly Gradient Modulation with Generalization Enhancement (OGM\_GE)~\cite{peng2022balanced}, and Learning Facilitator for Modality gap (LFM)~\cite{yang2024facilitating}. As shown in Fig.~\ref{fig2}(a1–a3), IEMF demonstrates consistent performance improvements across all fusion schemes and datasets.

Specifically, taking IEMF's enhancement to the MSLR method across datasets as an example, on the CREMA-D dataset, the baseline model achieved 64.11\% accuracy using MSLR, which improved to 65.59\% after introducing IEMF, yielding a 1.48\% gain. On the more challenging Kinetics-Sounds dataset, baseline accuracy was 51.89\%, while the IEMF-enhanced model reached 55.86\%, a 3.97\% improvement. Even on the UrbanSound8K-AV dataset where the baseline model already achieved high accuracy of 97.79\%, IEMF further improved it to 97.98\%. Though limited in magnitude, this improvement remains practically significant given the already high performance level.

It should be noted that in some cases, performance gains after introducing IEMF were relatively small, and in isolated instances even showed slight decreases (e.g., UrbanSound8K-AV dataset with Concat fusion strategy, Fig.~\ref{fig2}(a3)). This phenomenon can primarily be attributed to: when baseline models already optimally leverage complementary audio-visual information in clear, low-noise environments, the existing modal contribution ratios are already near-optimal, naturally diminishing the benefits of dynamic adjustment and occasionally introducing slight perturbations due to additional modeling freedom. Therefore, IEMF's performance improvement potential is relatively limited in high-baseline, low-interference environments; whereas in environments with fluctuating modal signal quality or noise interference, IEMF's adaptive regulation mechanism demonstrates more significant advantages.

Looking at overall trends, IEMF consistently improves performance across datasets and fusion strategies, confirming its effectiveness in enhancing multimodal fusion efficiency. Unlike traditional fusion methods, IEMF dynamically adjusts modal fusion module weights based on each modality's relative strength. When information in one modality (e.g., audio) decreases due to noise or distortion, IEMF promotes greater information compensation from the fusion module, improving overall perceptual accuracy and model robustness. This dynamic adaptive mechanism significantly enhances model robustness when facing input quality fluctuations and environmental uncertainties.

To further validate IEMF's effectiveness, we conducted mechanism ablation experiments (Fig.~\ref{fig2}(c1)-(c2)). In (c1), we analyzed classification accuracy changes under different inverse gain coefficient $\gamma$ settings. Results show appropriate inverse gain effectively improves model performance, with optimal accuracy at $\gamma=1$, indicating IEMF effectively balances unimodal and multimodal fusion signal contributions to maximize dynamic compensation. With larger coefficients (e.g., $\gamma=5$), accuracy decreases, likely due to training instability from excessive modulation intensity. Conversely, without IEMF (w/o IEMF baseline in Fig.~\ref{fig2}(c1)), classification accuracy is notably lower than all inverse gain coefficient settings, further validating the crucial role of IEMF in enhancing multimodal fusion.
In (c2), we compared baseline models (Vanilla), models with joint learning strategy, and IEMF-enhanced models. We specifically included joint learning comparison to systematically evaluate IEMF's effectiveness. Joint learning adds independent classification heads for each modality without introducing new modalities, enhancing unimodal feature discriminability. IEMF dynamically modulates fusion module updates based on unimodal-fusion signal strength relationships for more refined compensation—mechanistically different approaches. Results show that while joint learning provides performance improvements, IEMF further enhances model performance, validating IEMF's superior generalizability through dynamic fusion module adjustment in existing multimodal learning frameworks.

We further evaluated how multimodal learning affects performance of weaker modality branches (audio) on Kinetics-Sounds using OGM\_GE fusion (Fig.~\ref{fig2}(d)). After multimodal training, we fine-tuned the audio branch to analyze fusion effects on unimodal perception. Results show traditional fusion methods (orange curve) lead to overfitting, limiting performance gains and even underperforming independently trained unimodal branches (blue curve). This suggests modal interference in conventional fusion degrades unimodal feature quality and perception.
In contrast, with IEMF (green curve), the unimodal branch maintains higher, more stable accuracy throughout training with significant early performance advantages. This confirms IEMF not only optimizes multimodal fusion but effectively mitigates modal interference, promoting better unimodal feature learning and generalization.

Figure~\ref{fig2}(e) shows the test‐set evolution of IEMF dynamic coefficients $\xi$ and classification accuracy across training epochs. IEMF adapts fusion module behavior based on fusion effectiveness: during early training, fusion benefits are greater, keeping dynamic coefficient $\xi$ high to maximize multimodal advantages; as unimodal features mature and fusion advantages diminish, $\xi$ naturally decreases, reflecting reduced fusion dependency and helping maintain module stability for consistent test performance.

To validate the generalization properties of our proposed method, we visualized the loss landscapes of models with and without IEMF as shown in Fig.~\ref{fig2}(f). The 3D loss surface visualization illustrated in Fig.~\ref{fig2}(f1) reveals significant topological differences: the baseline method exhibits a sharper, cone-like minimum, while the IEMF-enhanced model displays a broader, more gradual basin structure. This distinction is further emphasized in the 2D contour plots depicted in Fig.~\ref{fig2}(f2-f3): without IEMF, the contours form elongated elliptical patterns, indicating inconsistent curvature across different parameter directions; with IEMF, contours appear more circular and uniformly spaced, confirming a significantly flatter minimum region. These observations closely align with our subsequent theoretical analysis presented later in this paper, which demonstrates that IEMF directs the optimization process toward flatter regions of the loss landscape, a characteristic directly associated with the improved generalization performance observed in our experimental results.

Beyond performance improvements, we analyzed IEMF's impact on computational cost as shown in Fig.~\ref{fig2}(g). Our evaluation employed a comprehensive computational cost metric that balances both convergence speed and per-epoch complexity, providing a more holistic assessment of algorithmic efficiency. Across all fusion methods, IEMF consistently reduces computational costs by significant margins. The computational savings range from 15.2\% for MSLR to 50.0\% for OGM\_GE, with Normal and LFM configurations showing reductions of 44.2\% and 36.6\%, respectively. These substantial improvements stem from IEMF's ability to achieve faster convergence while maintaining reasonable per-epoch complexity. By dynamically modulating fusion behavior based on modality contributions, IEMF effectively reduces the total computational budget required to reach optimal performance. Importantly, these efficiency gains occur concurrently with the performance enhancements reported earlier, demonstrating that IEMF not only improves model accuracy but also significantly optimizes computational resource utilization—a critical advantage for resource-constrained multimodal applications in real-world environments.

In summary, IEMF demonstrates consistent performance improvements across datasets and fusion strategies. Systematic experiments validate its effectiveness in dynamically regulating fusion, mitigating modal interference, enhancing unimodal learning, and improving robustness, providing an efficient and well-generalizing fusion strategy for multimodal perception tasks.

\begin{figure*}[t]
	\centering
		\includegraphics[width=0.85\linewidth]{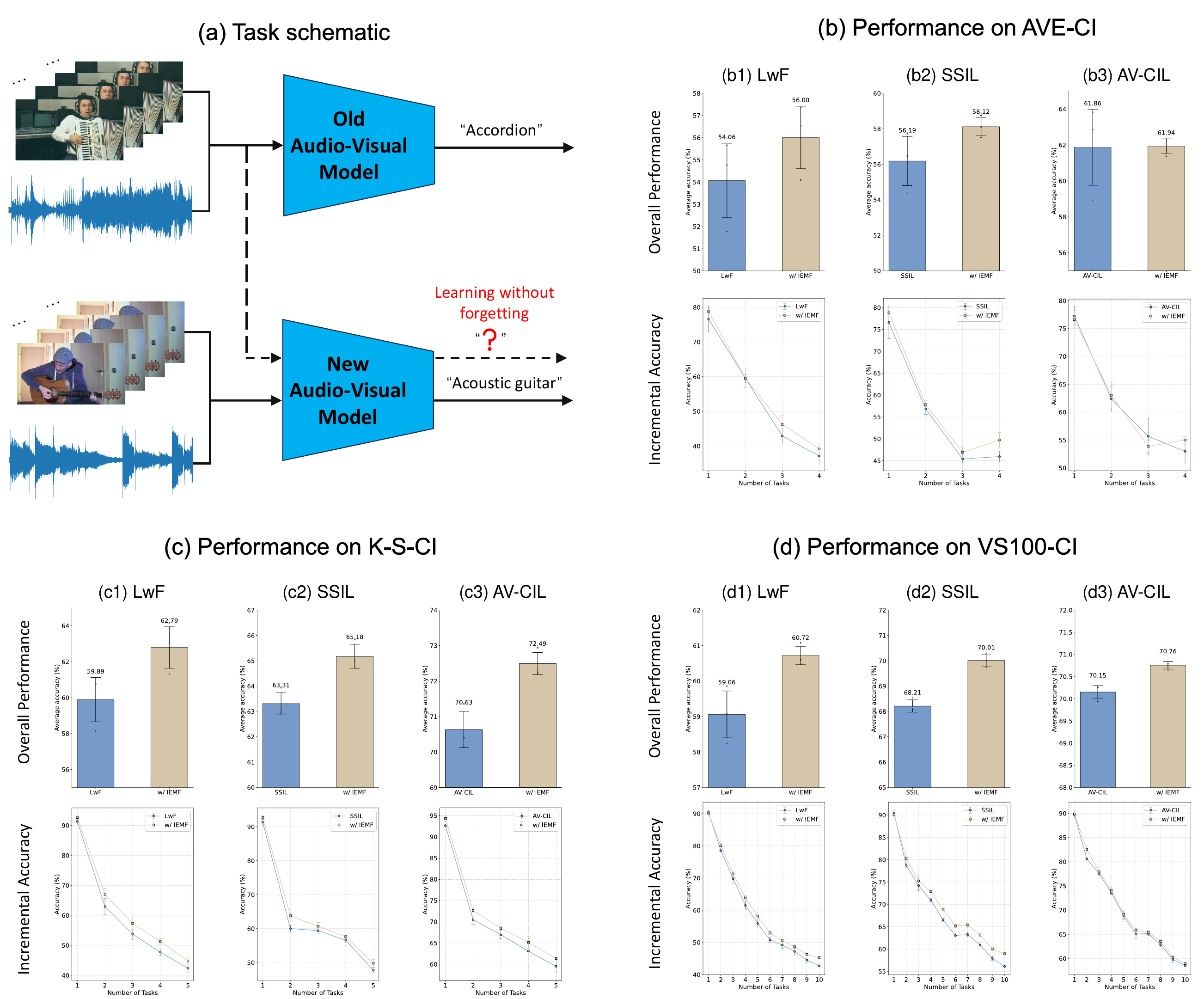}
    \caption{\textbf{Inverse effectiveness driven multimodal fusion boosts audio visual continual learning.}
    \textbf{(a) Task schematic.}  A single audiovisual model is incrementally updated as new classes arrive; the goal is to absorb the new knowledge while preserving performance on previously learned classes—achieving ``learning without forgetting''.  
    \textbf{(b) Results on AVE-CI}, \textbf{(c) K-S-CI} and \textbf{(d) VS100-CI}.  
    For three representative class incremental learning baselines—LwF, SSIL and AV-CIL—we compare the vanilla method (\textcolor{blue}{blue}) with the method augmented by IEMF (\textcolor{brown}{khaki}).  
    Each sub-panel is split into top and bottom:  
    the top bar chart reports the \emph{overall performance} (mean accuracy across all tasks, error bars denote one standard deviation),  
    while the bottom line plot traces the \emph{incremental accuracy} after each successive task.  
    Across all datasets and baselines, IEMF consistently increases mean accuracy and yields a flatter accuracy-decay curve, indicating the better knowledge transfer.}
	\label{fig3}
\end{figure*}

\bigskip
\noindent
\textbf{IEMF improved the model performance on audio visual continual learning}\par
\noindent
To evaluate the effectiveness of IEMF in more challenging scenarios, we further examined its performance in audio-visual continual learning tasks. In such tasks, models need to learn new categories continuously while preserving recognition capabilities for previously learned categories as much as possible, thereby avoiding catastrophic forgetting, as shown in Fig.~\ref{fig3}(a). We selected three representative class-incremental learning baseline methods for comparison: LwF~\cite{li2017learning}, SSIL~\cite{ahn2021ss}, and AV-CIL~\cite{pian2023audio}, and evaluated them on three audio-visual continual learning datasets: AVE-CI, K-S-CI, and VS100-CI~\cite{pian2023audio}. The experimental results are shown in Fig.~\ref{fig3}(b-d).

After introducing the IEMF method, the models achieved stable accuracy improvements across all datasets.
On AVE-CI, LwF increased from 54.06 \% to 56.00 \% (+1.94 \%), SSIL improved from 56.19 \% to 58.12 \% (+1.93 \%), and AV-CIL slightly increased from 61.86 \% to 61.94 \% (+0.08 \%).
In K-S-CI, which features more cross-modal noise, LwF rose from 59.89 \% to 62.79 \% (+2.90 \%), SSIL improved from 63.31 \% to 65.18 \% (+1.87 \%), and AV-CIL increased from 70.63 \% to 72.49 \% (+1.86 \%).
For the largest scale dataset VS100-CI, LwF improved from 59.06 \% to 60.72 \% (+1.66 \%), SSIL from 68.21 \% to 70.01 \% (+1.80 \%), and AV-CIL from 70.15 \% to 70.76 \% (+0.61 \%).
All nine comparisons showed positive gains, with an average improvement of approximately 1.63 percentage points, highlighting the consistent effectiveness of IEMF.

In Fig.~\ref{fig3}(b-d), the line graphs in the bottom row of each subfigure show the average accuracy changes after each continuous task. Notably, compared to baseline models, the accuracy decline curves of IEMF models are significantly more gradual. This indicates that IEMF enhances the model's ability to retain existing knowledge during cross-task knowledge transfer while effectively integrating information about new categories, thereby significantly mitigating catastrophic forgetting.

To further understand the internal mechanisms behind IEMF's performance improvements, we analyzed its fusion dynamic behavior during training. IEMF does not explicitly introduce learnable parameters bound to specific tasks, but rather adaptively regulates the update dynamics of the fusion module based on changes in the effectiveness of unimodal and multimodal signals, thereby implicitly adapting to modal variations across different task stages during continuous learning. Through the training process guided by the inverse effectiveness principle, the model can naturally adapt to changes in modal reliability during weight updates, thus reducing over-reliance on a single modality when perceptual conditions fluctuate. Benefiting from this adaptive optimization strategy, IEMF not only improves the average accuracy across all tasks but also maintains a smoother performance degradation curve, preserving high overall performance and cross-task knowledge coherence even as new tasks are continuously introduced.

\bigskip
\noindent
\textbf{IEMF improved the model performance on audio visual question answering}\par
\noindent
We further evaluated the effectiveness of IEMF in audio visual question answering (AVQA) tasks. In this task, models must answer text questions based on synchronized audio and video inputs, demanding higher capabilities for deep integration of multimodal information. As shown in Fig.~\ref{fig4}, radar charts ~\ref{fig4}(a1) and ~\ref{fig4}(b1) compare the classification accuracy of baseline models versus models with IEMF, and ST-AVQA~\cite{li2022learning} models versus models with IEMF, across different question types (audio-only questions, visual-only questions, and audio-visual combined questions).

\begin{figure*}[t]
	\centering
		\includegraphics[width=0.9\linewidth]{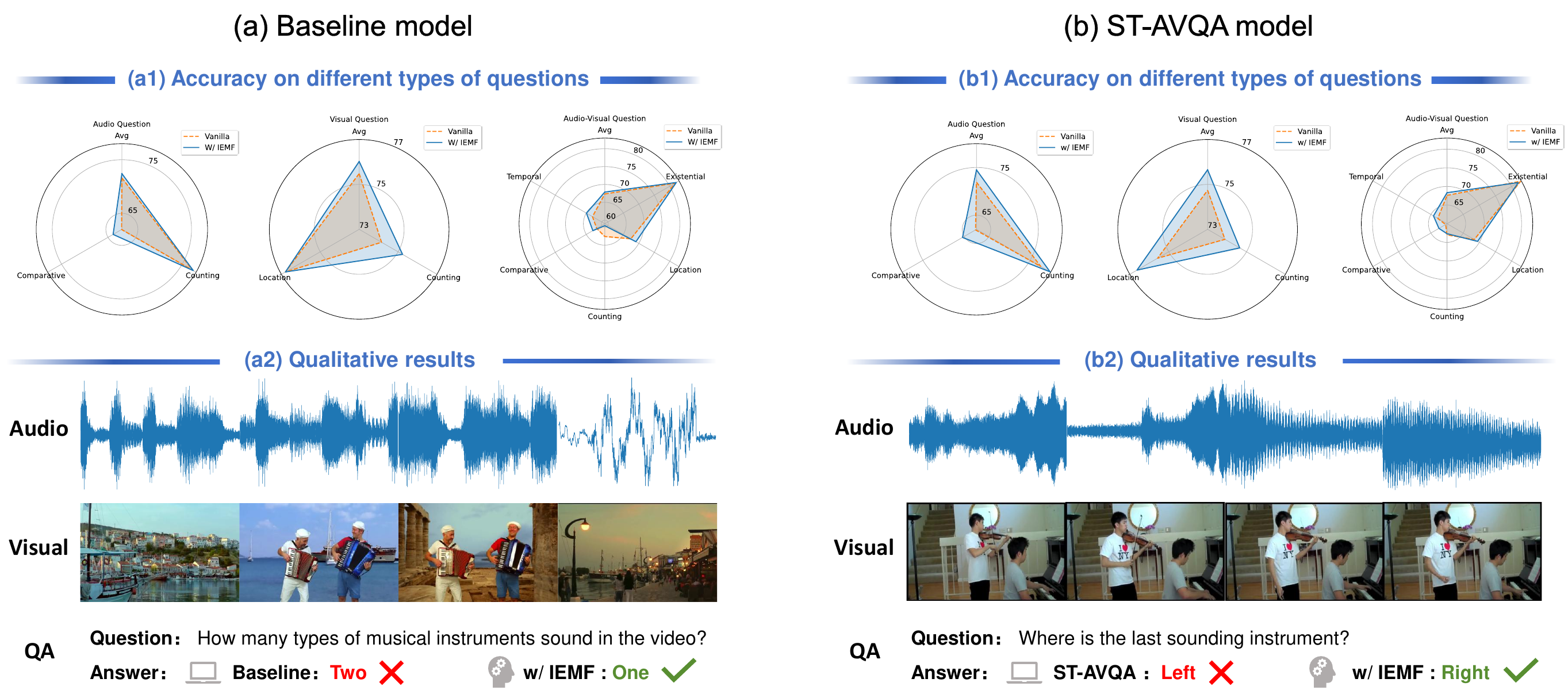}
    \caption{\textbf{Quantitative and qualitative impact of IEMF on audio visual question answering task.} \textbf{(a)} Baseline model.
    The three radar charts (top) report accuracy on audio-only, visual-only, and audio-visual questions, respectively.
    \textcolor{orange}{Orange} = vanilla, \textcolor{blue}{Blue} = w/ IEMF.
    The bottom row shows a representative sample—waveform, video frames, and question/answer—where the vanilla fusion miscounts the instruments (``Two''), whereas IEMF answers correctly (``One'').
    \textbf{(b)} ST-AVQA model.
    The same layout as (a), but using the stronger ST-AVQA model.
    IEMF again enlarges the radar area for every question type and corrects the localization query in the illustrated example (vanilla: ``Left''; w/ IEMF: ``Right'').
    Across both models, the blue polygons consistently enclose the orange ones, confirming that the inverse effectiveness driven multimodal fusion mechanism improves all question categories while providing intuitive per-sample gains.}
	\label{fig4}
\end{figure*}

Comparing the radar charts of original models and models with IEMF, we can observe that IEMF improved answer accuracy across all question types. Taking the ST-AVQA model and its IEMF-enhanced version as an example (Fig.~\ref{fig4}(b)), for audio-only questions, the original ST-AVQA model achieved an average accuracy of 71.90\%, while the IEMF model improved to 74.49\%, an increase of 2.59\%. Similarly, for visual-only questions, the baseline accuracy was 74.74\%, while the IEMF-enhanced model reached 75.65\%, an improvement of 0.91\%. For audio-visual questions, the vanilla model's average accuracy was 67.61\%, while the IEMF model achieved 68.33\%, an improvement of 0.72\%.
To verify IEMF's performance on fine-grained questions, we specifically analyzed its effectiveness in tasks requiring precise localization classification. As shown in Fig.~\ref{fig4}(b2), the original ST-AVQA model incorrectly predicted ``left'' side when answering ``the position of the last sounding instrument'', while the model with IEMF correctly located it as ``right'' side. This demonstrates that IEMF-enhanced models possess stronger fine-grained discrimination capabilities in complex cross-modal reasoning tasks, improving the integration efficiency of multimodal cues.

This improvement further validates IEMF's crucial advantages in handling noisy interference or incomplete input information scenarios. By dynamically adjusting the update rate of the fusion module during training based on the strength of unimodal and multimodal fusion signals, IEMF guides the model to learn strategies that can more robustly integrate different modal information when modal signal strengths are uneven or information is contradictory. In contrast, models without IEMF are more prone to judgment biases when facing modal conflicts or input uncertainties, leading to incorrect answers or overall performance degradation. Overall, these results highlight IEMF's important role in enhancing multimodal understanding and reasoning.

\section{Discussion}
\noindent
\textbf{Biological insights into multimodal integration}\par
\noindent
Despite significant advances in multimodal fusion, many key biological principles have not been fully explored and applied in artificial intelligence systems, which could further enhance the robustness and adaptability of multimodal systems. In this study, we propose a brain-inspired deep neural network multimodal integration method based on the inverse effectiveness mechanism observed in biological systems, providing verifiable theoretical foundations and methodological support for multimodal information integration. Specifically, we propose an inverse effectiveness driven multimodal fusion (IEMF) method that dynamically adjusts the weights of modal fusion modules based on the relationship between the strength of single-modal cues and the signal strength after modal fusion. Our approach systematically considers the complementary interactions between modalities, significantly improving not only the performance and generalization capabilities of multimodal systems but also their computational efficiency. This dual advantage—enhanced robustness coupled with reduced computational costs—offers insights into why inverse effectiveness might have evolved as a critical mechanism in biological systems, where both perceptual reliability and metabolic efficiency are under evolutionary pressure.

In the IEMF framework, we prioritize inverse effectiveness in the network training process by introducing an inverse effectiveness gain coefficient that applies gradient regulation to fusion weights, naturally forming an internal bias of ``weak modality-high gain, strong modality-low gain'' during the learning phase. This strategy aligns with the physiological mechanism of cross-modal experience shaping plasticity in early neural development in biological organisms, where newborn individuals initially lack multisensory integration abilities. These abilities are not innate but develop through continuous shaping of neural circuits via early cross-modal experiences, adapting to the environment and optimizing multimodal integration performance~\cite{stein2014development}. Furthermore, we apply inverse effectiveness driven fusion strategies uniformly across all input channels of the model, consistent with the view in~\cite{regenbogen2018intraparietal} that inverse effectiveness manifests not only under degraded stimulus conditions but also with clear stimuli.

This research confirms the long-standing intuition that how modalities are combined is as important as how many are combined. By introducing inverse effectiveness rules from cortical circuits into gradient-based optimization learning systems, we achieved: (i) effective generalization in both Artificial Neural Networks (ANN) and Spiking Neural Networks (SNN) models; (ii) significant performance improvements in audio-visual classification, audio-visual continual learning and audio-visual question answering tasks; and (iii) substantial computational efficiency gains with up to 50\% reduction in computational costs across diverse fusion methods. Systematic experiments demonstrate that after integrating the IEMF mechanism into existing multimodal methods, models achieved performance superior to original state-of-the-art techniques across various multimodal tasks, further indicating that introducing bio-inspired mechanisms can effectively improve the efficiency of multimodal integration, expanding its potential in artificial intelligence applications. These findings not only highlight the advantages of incorporating biological principles into machine learning models but also provide new directions for future research in neuromorphic computing and multisensory integration.

\bigskip
\noindent
\textbf{Other biological mechanisms of multimodal integration}\par
\noindent
It is worth discussing that while this research emphasizes the importance of inverse effectiveness in multimodal integration, it is worth noting that there are two other equally important principles in multimodal biological perception processes: temporal congruence and spatial congruence. These mechanisms are particularly important in dynamic multimodal integration.
Temporal congruence refers to visual and auditory inputs maintaining coordination in time, thereby optimizing perceptual and decision-making performance. Experimental studies~\cite{slutsky2001temporal, hairston2006auditory} demonstrate that when visual and auditory stimuli are presented in close synchronization within a 0-200 millisecond time window, they significantly enhance the accuracy and reaction speed of perceptual judgments. In contrast, temporal asynchrony leads to decreased activation intensity in relevant brain regions, weakening the integration effect.
Spatial congruence refers to different sensory modalities maintaining consistency or proximity in spatial location, thereby enhancing the joint representation of cross-modal signals. Research has found that multisensory neurons (such as neurons in the superior colliculus) exhibit integration enhancement effects only when audiovisual stimuli originate from the same or adjacent spatial locations; otherwise, integration may be inhibited or show no integration response~\cite{stein1993merging}. For example, in real-world scene understanding, object recognition, and tracking tasks, accurately matching sound sources with corresponding visual objects is key to successful multimodal perception. Temporal and spatial congruence are crucial for accurate multimodal integration.

Although temporal and spatial congruence are indispensable in biological perception, given that the tasks selected in this study inherently possess strong input synchronization characteristics (i.e., dual-modal inputs from the same source at the same moment), we did not explicitly model these mechanisms in the current work. Specifically, the sensory inputs in this study's tasks naturally possess synchronization and correspondence relationships; therefore, these congruence factors have already been implicitly considered in the multimodal fusion process.
Looking forward, further research could explore how to explicitly incorporate temporal and spatial congruence into the IEMF framework by introducing asynchronous, spatially disparate multimodal input samples, thereby training models to effectively integrate under more complex temporal and spatial variation conditions, further advancing biologically inspired multimodal learning systems toward broader application domains.

\section{Materials and methods}
\noindent
\textbf{Neuron models in ANNs and SNNs}\par

\noindent
In neural networks, the information flow is governed by the dynamics of neuronal activation. In this work, we adopt two distinct neuron models: the continuous artificial neurons used in artificial neural networks (ANNs) and the spike-based neurons in spiking neural networks (SNNs).


In ANNs, information is processed continuously. Each neuron computes a weighted sum of its inputs and applies a nonlinear activation function to produce its output: $ y = f(Wx + b)$, where $x$ is the input vector, $W$ is the weight matrix, $b$ is the bias term, and $f(\cdot)$ denotes a nonlinear function such as ReLU~\cite{nair2010rectified} or sigmoid.

In contrast, SNNs more closely mimic biological neurons by communicating via discrete spike events. We employ the widely used Leaky Integrate-and-Fire (LIF) model~\cite{dayan2005theoretical} to capture the membrane potential dynamics. Upon receiving a synaptic input current $I(t)$, the membrane potential $U(t)$ accumulates over time. When $U(t)$ crosses a threshold $U_{\text{th}}$, a spike is emitted and the potential is reset to the resting value $U_{\text{rest}}$. The continuous-time dynamics of the LIF neuron are given by: 
\begin{equation} 
  \tau_m \frac{d U(t)}{dt} = -\left(U(t) - U_L\right) - \frac{g_{E \mid I}}{g_L} (U(t) - U_{E \mid I}) + \frac{I_{\text{s}}}{g_L}, 
\end{equation} 
where $\tau_m = C_m / g_L$ is the membrane time constant, $C_m$ is membrane capacitance, $g_L$ is the leak conductance, and $g_{E \mid I}$ and $U_{E \mid I}$ denote the conductance and reversal potentials for excitatory or inhibitory synapses. $I_{\text{s}}$ is the synaptic input current.
To simplify the formulation, we aggregate the synaptic terms into an effective input current: $R I(t) \triangleq - \frac{g_{E \mid I}}{g_L} (U(t) - U_{E \mid I}) + \frac{I_{\text{s}}}{g_L}$, reducing the membrane potential equation to: 
\begin{equation} 
  \tau_m \frac{dU(t)}{dt} = - (U(t) - U_{\text{rest}}) + R I(t). 
\end{equation}
For numerical simulation, we set $U_{\text{rest}} = 0$ and discretize the above dynamics. To clearly distinguish continuous and discrete states, we denote membrane potential as $\mathbf{u}^t$ at discrete time step $t$. The complete discrete-time update of the membrane potential and spike generation at layer $l$ is: 
\begin{equation}
  \left\{
  \begin{aligned}
  &\mathbf{u}_{\text{pre}}^{t, l} = \tau \mathbf{u}^{t-1, l} + \mathbf{W}^l \mathbf{s}^{t, l-1}, &\footnotesize{\text{(accumulation)}} \\
  &\mathbf{s}^{t, l} = H\left( \mathbf{u}_{\text{pre}}^{t, l} - \mathbf{u}_{\text{th}} \right), &\footnotesize{\text{(spike firing)}} \\
  &\mathbf{u}^{t, l} = \mathbf{u}_{\text{pre}}^{t, l} \cdot \left(1 - \mathbf{s}^{t, l} \right), &\footnotesize{\text{(reset mechanism)}}
  \end{aligned}
  \right.
  \end{equation}
where $\mathbf{W}^l$ is the weight matrix from layer $l-1$ to $l$, $\mathbf{s}^{t, l-1}$ denotes the spike train from the previous layer at time $t$, $\tau = 1 - \frac{1}{\tau_m}$ is the leak factor controlling the temporal decay of the membrane potential, and $H(\cdot)$ is the Heaviside step function used to generate binary spike outputs.

\bigskip
\noindent
\textbf{Multimodal integration formulation}\par
\noindent
We denote a multimodal input as $\boldsymbol{x} = (\boldsymbol{x}^a, \boldsymbol{x}^v)$, where $\boldsymbol{x}^a \in \mathcal{X}^a$ and $\boldsymbol{x}^v \in \mathcal{X}^v$ represent inputs from two modalities (e.g., audio and visual). These are independently processed by two encoders $\varphi^a(\cdot; \boldsymbol{\theta}^a)$ and $\varphi^v(\cdot; \boldsymbol{\theta}^v)$ to obtain modality-specific latent representations: 
\begin{equation} 
  \boldsymbol{z}^a = \varphi^a(\boldsymbol{x}^a; \boldsymbol{\theta}^a), \quad \boldsymbol{z}^v = \varphi^v(\boldsymbol{x}^v; \boldsymbol{\theta}^v), 
\end{equation} 
where $\boldsymbol{\theta}^a$ and $\boldsymbol{\theta}^v$ represent the trainable parameters of the audio and visual encoders, respectively.
The extracted features are fused using a general fusion operator $\mathcal{F}(\cdot, \cdot)$, i.e., 
$
  \boldsymbol{z}^{\text{av}} = \mathcal{F}(\boldsymbol{z}^a, \boldsymbol{z}^v), 
$
followed by a classifier $h(\cdot; \boldsymbol{\theta}^h)$ that maps the fused audio visual features to a prediction
$
  \hat{\boldsymbol{y}} = h(\boldsymbol{z}^{\text{av}}; \boldsymbol{\theta}^h), 
$
where $\boldsymbol{\theta}^h$ denotes the classifier parameters.

In the widely adopted vanilla fusion strategy, such as feature concatenation, the fusion operator takes the form: 
\begin{equation} 
  \mathcal{F}(\boldsymbol{z}^a, \boldsymbol{z}^v) = \mathbf{W}^{f} \left[ \boldsymbol{z}^a ; \boldsymbol{z}^v \right] + \mathbf{b}^{f}, 
\end{equation} 
where $[\cdot ; \cdot]$ denotes the concatenation operation, $\mathbf{W}^{f}  \in \mathbb{R}^{M \times (d_a + d_v)}$ and $\mathbf{b}^{f} \in \mathbb{R}^M$ are the parameters of the fusion layer, and $M$ is the number of output classes.

The multimodal learning goal is to train a multimodal model $f_{\boldsymbol{\theta}}: \mathcal{X}^a \times \mathcal{X}^v \rightarrow \mathcal{Y}$, where the learnable parameters $\boldsymbol{\theta} = \{\boldsymbol{\theta}^a, \boldsymbol{\theta}^v, \boldsymbol{\theta}^h\}$, that minimizes the empirical risk over a dataset $\mathcal{D} = {(\boldsymbol{x}_i^a, \boldsymbol{x}_i^v, y_i)}_{i=1}^N$. 
The training objective is: 
\begin{equation} 
  \arg\min_{\boldsymbol{\theta}} \; \mathcal{L}(\boldsymbol{\theta}) = \frac{1}{N} \sum_{i=1}^N \mathcal{L}_{\text{ce}}\left(f_{\boldsymbol{\theta}}(\boldsymbol{x}_i^a, \boldsymbol{x}_i^v), y_i\right), 
\end{equation} 
where $\mathcal{L}_{\text{ce}}(\cdot)$ denotes the cross-entropy loss function.

\bigskip
\noindent
\textbf{Inverse effectiveness driven multimodal fusion}\par
\noindent
Most previous studies have focused on designing sophisticated fusion blocks, refining the joint representation $\boldsymbol{z}^{\text{av}}$ to enhance cross‑modal interactions, yet they rarely consider how the relative informativeness of each unimodal stream relative to the fused output should guide the fusion module. In this work, we draw inspiration from the principle of inverse effectiveness. 
Instead of employing static cross‑modal integration, we adaptively adjust the update rate of the fusion module’s weights by contrasting the estimated information content of each unimodal branch with that of the integrated multimodal representation.

First, for each sample $i$ in a mini‑batch $\mathcal{B}_t$, we evaluate the per‑sample modal information content $c_i$ according to the following equation:
\begin{equation}
  \begin{aligned}  
    c_i^{a} &= \bigl[\mathbf{p}_i^{a}\bigr]_{y_i},  
    \quad &\mathbf{p}_i^{a} &=  
    \boldsymbol{\pi}\!\Bigl(\mathbf{W}_t^{a} \cdot \boldsymbol{z}^a + \mathbf{b}_t^a\Bigr),\\  
    c_i^{v} &= \bigl[\mathbf{p}_i^{v}\bigr]_{y_i},  
    \quad &\mathbf{p}_i^{v} &=  
    \boldsymbol{\pi}\!\Bigl(\mathbf{W}_t^{v} \cdot \boldsymbol{z}^v + \mathbf{b}_t^v\Bigr),  
    \end{aligned}
    \label{ic-1}
\end{equation}
where $\mathbf{W}_t^{a/v}$, $\mathbf{b}_t^{a/v}$ are the parameters of the audio/visual modal classification heads, respectively,
and $[\mathbf{p}]_{y_i}$ picks the probability assigned to the ground‑truth label $y_i$. $\boldsymbol{\pi}$ is a normalization function, and here we choose the softmax function. The informativeness of the multimodal output is estimated in the same way:
\begin{equation}
  c_i^{\text{av}} = \bigl[\mathbf{p}_i^{\text{av}}\bigr]_{y_i},  
  \quad \mathbf{p}_i^{\text{av}} =  
  \boldsymbol{\pi}\!\Bigl(\mathbf{W}_t^{av} \cdot \boldsymbol{z}^{\text{av}} + \mathbf{b}_t^{av}\Bigr),
  \label{ic-2}
\end{equation}
Next, we average the evaluated values in \eqref{ic-1} and \eqref{ic-2} to obtain the batch‑level modality‑strength scores.
\begin{equation}
  S_t^{a-v}=\frac{1}{2|\mathcal{B}_t|}
  \sum_{i\in\mathcal{B}_t}\bigl(c_i^{a}+c_i^{v}\bigr),
  \quad
  S_t^{av}=\frac{1}{|\mathcal{B}_t|}
  \sum_{i\in\mathcal{B}_t} c_i^{av}.
  \end{equation}
Following the biological observation that fusion should dominate, and provides the greatest benefit when unimodal evidence is weak, we define a inverse effectiveness driven multimodal fusion coefficient $\xi_t$.
The IEMF coefficient quantifies how effective the unimodal branches are relative to the fused output. We map $\xi_t$ to a bounded value:
\begin{equation}
  \label{eq:kappa}
  \xi_t = \gamma \cdot \left(1 + \kappa\,\bigl(1- \frac{S_t^{a-v}}{S_t^{av}}\bigr)\right),
  \qquad
  \gamma>0,
  \end{equation}
  where $\gamma$ is the inverse gain coefficient controlling the overall magnitude of fusion modulation, and $\kappa(\cdot)$ denotes a generic bounded gating function, in this work we instantiate $\kappa$ with the hyperbolic tangent, i.e., $\kappa(\cdot)=\tanh(\cdot)$, owing to its smooth and symmetric saturation properties. Because $\kappa(\cdot)\!\in(-1,1)$, Eq.~\eqref{eq:kappa} confines the fusion  
  coefficient to the interval $\xi_t\in(0,2\gamma)$.  

  IEMF coefficient $\xi_t$ magnitude varies intuitively with the strength ratio between unimodal and multimodal evidence:
  \begin{itemize}
    \item \emph{Weak unimodal evidence} (\(S_t^{a-v} \ll S_t^{av}\)).  
          The term \(1-S_t^{\text{a-v}}/S_t^{av}\) is positive and thus \(\xi_t\) approaches its upper bound.  
          A larger \(\xi_t\) amplifies the fusion gradient, encouraging the
          model to rely more heavily on cross‑modal integration.
    \item \emph{Strong unimodal evidence} (\(S_t^{a-v} \gtrsim S_t^{av}\)).  
          As the ratio nears or exceeds 1, the inner term becomes
          non‑positive; \(1 + \kappa(\cdot)\) decreases, pulling
          \(\xi_t\) toward its lower limit 0.  
          When the unimodal score predominates, a smaller $\xi_t$ attenuates the fusion update, preserving the integrity of an already robust unimodal pathway.
    \end{itemize}
    In summary, a large \(\xi_t\) corresponds to weak unimodal cues and
    triggers a stronger adjustment of the fusion weights, whereas a
    small \(\xi_t\) indicates confident unimodal predictions and
    results in a milder fusion update.

    The inverse effectiveness driven multimodal fusion coefficient \(\xi_t\) is applied only to
    the fusion parameters and the unimodal branches are not affected.
    Concretely,
    \begin{equation}
        \mathbf{W}_{t+1}^{f} = \mathbf{W}_{t}^{f} - \eta\,\xi_t\,
        \nabla_{\mathbf{W}^{f}}\mathcal{L}(\mathbf W_t^f),
        \label{eq:update}
    \end{equation}
  where \(\eta\) is the learning rate and \(\nabla_{\mathbf{W}^{f}}\mathcal{L}(\mathbf W_t^f)\) is the raw fusion gradient.
    Because \(0<\xi_t<2\gamma\), this scaling never reverses the descent
    direction, thereby maintaining optimisation stability.  
    Together, Eqs.\,\eqref{ic-1}–\eqref{eq:update}, inspired by the inverse effectiveness principle, are characterized as: the fusion pathway receives a
    larger update when unimodal evidence is weak and a smaller
    one when unimodal confidence is high.  This self‑balancing rule enhances
    robustness under noise and improves the model's generalization across diverse input conditions.

\bigskip
\noindent
\textbf{Theoretical analysis of inverse effectiveness driven multimodal fusion strategy}\par
\noindent
We prove that the inverse‑effectiveness coefficient~$\xi_t$
used by IEMF reduces the expected step size more
in high‑curvature directions, ensuring reliable convergence to local minima while maintaining optimization stability throughout the training process.
\begin{assumption}\label{asm:smooth}
The loss $\mathcal L(\mathbf W^f)$ is twice continuously differentiable
and there exist constants $\beta,\rho>0$ such that for all
$\mathbf u,\mathbf v$,
$\|\nabla\mathcal L(\mathbf u)-\nabla\mathcal L(\mathbf v)\| \le\beta\|\mathbf u-\mathbf v\|$,
$\|\mathbf H(\mathbf u)-\mathbf H(\mathbf v)\| \le\rho\|\mathbf u-\mathbf v\|.$
This means that the gradient and the Hessian are $\beta$-Smoothness and $\rho$-Lipschitz respectively.
\end{assumption}

\begin{theorem}[Convergence properties of IEMF]
  \label{thm:iemf_convergence}
Assume \ref{asm:smooth}.  Let $\mathbf W^{f*}$ be a local minimizer and
write $\mathbf H^*=\mathbf H(\mathbf W^{f*})$ with eigen‑pairs
$\{(\lambda_i,\mathbf e_i)\}_{i=1}^{d}$,
$0<\lambda_1\le\dots\le\lambda_d$.
The IEMF updates the fusion module parameters according to the equation \eqref{eq:update}.
Define the deviation
\(\Delta_t:=\mathbf W_t^f-\mathbf W^{f*}
          =\sum_{i=1}^d\alpha_i^t\mathbf e_i\).
Choose a radius \(r>0\) such that
\[
\frac{\rho}{2}r < \lambda_1
\quad\text{and}\quad
\|\Delta_t\|\le r\;\;\forall t .
\]
Then,
\[
\mathbb E[\eta\xi_t\lambda_i]
\begin{cases}
<\eta\lambda_i, & S_t^{a‑v}/S_t^{av}>1 \quad\text{(unimodal dominated)},\\
=\eta\lambda_i, & S_t^{a‑v}/S_t^{av}=1,\\
>\eta\lambda_i, & S_t^{a‑v}/S_t^{av}<1 \quad\text{(fusion dominated)}.
\end{cases}
\]
As a result, with IEMF, unimodal‑dominated batches reduce sharp directions more than vanilla method, whereas fusion‑dominated batches allow at most a two‑fold increase in step size, thus preserving optimization convergence.
\end{theorem}

\begin{proof}
  For any $\mathbf{W}^f$ satisfying $\|\mathbf{W}^f-\mathbf W^{f*}\|\le r$, we use $\nabla\mathcal{L}(\mathbf{W})$ denotes $\nabla_{\mathbf W^f}\mathcal{L}(\mathbf{W}^{f})$ and construct
  \(
  g(s)=\nabla_{\mathbf W^f} \mathcal{L}\left(\mathbf{W}^{f *}+s \Delta\right), s \in[0,1].
  \)
  The fundamental theorem of calculus gives
  \(
  \nabla \mathcal{L}(\mathbf{W}) = g(1)-g(0)
  = \int_0^1 \frac{d}{ds} g(s)  ds.
  \)
  Then we have:
  \begin{equation*}
  \nabla \mathcal{L}(\mathbf{W}) = \int_0^1 \mathbf{H}\left(\mathbf{W}^{f }+s \Delta\right) \Delta ds.
\end{equation*}
Adding and subtracting $\mathbf{H}^* \Delta$ inside the integral yields:
  \begin{align*}
  \nabla \mathcal{L}(\mathbf{W}) &= \mathbf{H}^* \Delta + \int_0^1\left[\mathbf{H}\left(\mathbf{W}^{f*}+s \Delta\right)-\mathbf{H}^*\right] \Delta \ ds.
  \end{align*}
  We define the remainder term as:
  \begin{equation*}
  \mathbf{R}(\mathbf{W}) = \int_0^1\left[\mathbf{H}\left(\mathbf{W}^{f*}+s \Delta\right)-\mathbf{H}^*\right] \Delta ds,
  \end{equation*}
  where the remainder $\mathbf{R}(\mathbf{W})$ can be estimated for an upper bound with the help of the $\rho$-Lipschitz condition
  \begin{equation}
    \label{eq:taylor}
  \|\mathbf R(\mathbf W)\|\le\tfrac{\rho}{2}\|\mathbf W-\mathbf W^{f*}\|^{2}.
  \end{equation}
  which allows us to express the gradient as:
  \begin{equation*}
  \nabla \mathcal{L}(\mathbf{W}) = \mathbf H^*(\mathbf W-\mathbf W^{f*}) + \mathbf{R}(\mathbf{W}).
  \end{equation*}
Insert the above equation into the update \eqref{eq:update}, and
Take inner product with $\mathbf e_i$ and use
$(\mathbf H^*\mathbf e_i)^{\!\top}=(\lambda_i\mathbf e_i)^{\!\top}$, we have
\begin{equation}
  \label{eq:recursion}
\alpha_i^{t+1}
 = (1-\eta\xi_t\lambda_i)\,\alpha_i^t
   - \eta\xi_t\,\mathbf e_i^{\!\top}\mathbf R(\mathbf W_t^f),
\end{equation}
where $\alpha_i^t=\mathbf e_i^{\!\top} \Delta_t$.
For the contraction argument we need the last term  
$\eta\xi_t\,|\mathbf e_i^{\!\top}\mathbf R(\mathbf W_t^f)|$ to be
strictly smaller than the linear part
\(\lambda_i\|\Delta_t\|\) even in the flattest direction
(\(\lambda_i=\lambda_1\)).
With Cauchy–Schwarz formula and \eqref{eq:taylor}, we have
  \begin{align*}
    |\mathbf e_i^{\!\top}\mathbf R(\mathbf W_t^f)|  
    &\;\le\;  
    \|\mathbf R(\mathbf W_t^f)\|, \\
    \|\mathbf R(\mathbf W_t^f)\|  
    &\;\le\;  
    \frac{\rho}{2}\|\Delta_t\|^{2}.   
\end{align*}
Applying the triangle inequality to Eq. \eqref{eq:recursion} and substituting our derived bounds, we have
\begin{align*}
  |\alpha_i^{t+1}|  
  &\;\le\;  
  |1-\eta\xi_t\lambda_i|\;|\alpha_i^{t}|  
  +\eta\xi_t\bigl(\tfrac{\rho}{2}\|\Delta_t\|^{2}\bigr)
\end{align*}
By imposing $\frac{\rho}{2}r < \lambda_1$, we establish a crucial inequality:
\[
  \eta\xi_t\bigl(\tfrac{\rho}{2}\|\Delta_t\|^{2}\bigr)
  \;\le\;\eta\xi_t\frac{\rho}{2}\,r\,\|\Delta_t\|
  \;<\;\eta\xi_t\lambda_1\,\|\Delta_t\|.
\label{eq:remainder_bound}
\]
This inequality demonstrates that the quadratic remainder term is always strictly dominated by the linear term for all eigendirections $i$, since $\lambda_i \geq \lambda_1$ for all $i$.
Consequently, the convergence behavior of each component $\alpha_i^t$ is primarily determined by the multiplicative factor $(1-\eta\xi_t\lambda_i)$, with the remainder term providing a bounded perturbation that does not disrupt the overall convergence pattern established by the linear term.
$\xi_t$ is computed by equation \eqref{eq:kappa}, 
with $\gamma=1$, \(0<\xi_t<2\).
Taking expectation of $\eta\xi_t\lambda_i$ over the mini‑batch three
cases arise, yielding exactly the inequalities stated in the theorem.
\end{proof}

Our analysis reveals two key aspects of IEMF's convergence properties, primarily manifested through the factor $(1-\eta\xi_t\lambda_i)$, which precisely controls how quickly the error components $\alpha_i^t$ (representing the projection of parameter error onto each eigenvector) contract toward zero.
In unimodal-dominated batches ($S_t^{a‑v}/S_t^{av}>1$ resulting in $\xi_t<1$), the contraction factor satisfies $|1-\eta\xi_t\lambda_i|<|1-\eta\lambda_i|$, meaning that high-curvature directions (larger $\lambda_i$ values) contract faster than in vanilla method. 
Meanwhile, in fusion-dominated batches ($S_t^{a‑v}/S_t^{av}<1$ resulting in $\xi_t\in(1,2)$), although step sizes may increase, they remain strictly bounded since $\xi_t < 2$, ensuring the algorithm's global stability. This dual mechanism, which balances a preference for minima with strict step-size constraints, ensures that IEMF maintains reliable convergence properties while adaptively adjusting step sizes.

Empirical studies have demonstrated a strong correlation between flatter minima and improved generalization performance~\cite{keskar2017large,ForetKMN21}. While Theorem~\ref{thm:iemf_convergence} establishes the convergence of IEMF under standard smoothness assumptions, a stronger theoretical link between the geometric properties of the solution and its generalization remains challenging to formally prove. Nevertheless, we provide the following heuristic justification based on landscape sharpness analysis, supported by experimental observations.

    Consider the sharpness of the loss landscape at a parameter configuration $\mathbf{W}^f$, defined as
    \begin{equation}
    s(\mathbf{W}^f, \rho) := \max_{|\boldsymbol{\epsilon}|_2 \leq \rho} \mathcal{L}(\mathbf{W}^f + \boldsymbol{\epsilon}) - \mathcal{L}(\mathbf{W}^f),
    \end{equation}
    which quantifies the sensitivity of the loss to local perturbations.
    While we do not provide a formal guarantee, our empirical evidence and directional analysis suggest the following heuristic conclusion:
    
    \begin{heuristic}[IEMF reduces landscape sharpness]
    Under dynamic training, IEMF adaptively reduces the step size in high-curvature directions. As a result, the sharpness
    \begin{equation}
    \mathbb{E}[s(\mathbf{W}^f, \rho)] \lessapprox \alpha \cdot s_{\mathrm{vm}}(\mathbf{W}^f, \rho),
    \end{equation}
    where $s_{\mathrm{vm}}$ is the sharpness observed under vanilla method and $\alpha < 1$ is a factor that quantifies how much IEMF reduces the loss landscape's sharpness through its adaptive modulation of optimization steps.
    \end{heuristic}
    
    This suggests that IEMF biases the optimization trajectory toward flatter regions of the loss landscape, a property that empirically correlates with improved generalization.

\bigskip
\noindent
\textbf{Experimental settings and training details}\par
\paragraph{Datasets. }
We divided the dataset as specified in the original dataset.
Audio‑visual classification: CREMA‑D~\cite{cao2014crema}, an audiovisual dataset containing six most common emotion categories for speech emotion recognition with total 7442 video clips. We randomly divided the dataset into a training and validation set, as well as a test set, with a ratio of 9:1; 
Kinetics-Sounds~\cite{arandjelovic2017look}, contains 31 human action categories selected from the Kinetics dataset~\cite{kay2017kinetics}. The dataset contains 17,366 10-second video clips, of which 1,472 are training and validation samples and 2,594 are test samples; UrbanSound8K-AV dataset~\cite{guo2023transformer}, with 8732 audiovisual samples totaling 10 categories. Each sample consists of a color image and a 4-second audio signal. We randomly divided the dataset into training and test sets in the ratio of 7:3.
Audio Visual Continual learning: We used the class-incremental audiovisual dataset introduced by~\cite{pian2023audio}, comprising three benchmark datasets: AVE‑CI (4 tasks $\times$ 7 classes) consisting of 3,294 training samples, 391 validation samples, and 394 test samples, K‑S‑CI (5 tasks $\times$ 6 classes) containing 19,220 training samples, 1,947 validation samples, and 1,958 test samples and VS100‑CI (10 tasks $\times$ 10 classes) with 51,195 training samples, 5,000 validation samples, and 5,000 test samples.
For audio‑visual question answering, we used the official MUSIC‑AVQA~\cite{li2022learning} split, which contains 32,087 training, 4,595 validation, and 9,185 test question–answer pairs. 

\paragraph{Data processing and network backbones. }
All raw videos were first resampled to a uniform frame rate. According to different task settings, we randomly sampled 1, 3, or 16 frames from each video clip as the visual input. The corresponding audio input was transformed into log-Mel spectrograms, which were used as input to the audio branch.
For audiovisual classification, we employed the ResNet-18~\cite{he2016deep} architecture for both the visual and audio streams. To investigate architectural generality, we adapted this topology to a spiking neural network counterpart by replacing conventional activation functions with leaky integrate-and-fire (LIF) neurons; neuron‑level hyper‑parameters are detailed in Table~\ref{tab:snn_ann_parameters}. For audiovisual continual learning, we used VideoMAE~\cite{tong2022videomae} and AudioMAE~\cite{huang2022masked} to extract video frames and audio features. For audio-visual question and answer, for vision we used pre-trained ResNet-18 model and for audio we used pre-trained VGGish~\cite{hershey2017cnn} to extract visual and audio features respectively.

\paragraph{Optimization details. } 
We trained models with stochastic gradient descent (SGD)~\cite{robbins1951stochastic} and a weight‑decay coefficient of $1\times10^{-4}$.  
For the audio visual classification setting we ran 100 epochs with an initial learning rate of $5\times10^{-3}$ and a mini‑batch size of 32.  
In the audio visual continual learning task, each incremental task was also trained for 100 epochs, but with a higher learning rate of $1\times10^{-2}$ and a batch size of 256 to accommodate the larger episodic memory.  
Audio visual question answering models were trained for 50 epochs using the same learning rate ($1\times10^{-2}$) and a batch size of 64.  
On the VS100-CI benchmark, where gradient noise is significant, we replaced SGD with Adam~\cite{KingmaB14} ($\beta_{1}=0.9$, $\beta_{2}=0.999$) to ensure smoother convergence.

\paragraph{Experimental platform. }
All experiments were conducted on a linux server equipped with NVIDIA A100-40 GB GPUs and an AMD EPYC 7763 processor. 

\bigskip
\noindent
\textbf{Details of evaluation metrics}\par
\noindent
We evaluate the proposed method across three multimodal tasks using task-specific metrics.

(1) Audio-Visual Classification. For audio-visual classification tasks, we report the standard Top-1 accuracy,
\(
Acc = \tfrac{1}{N} \sum_{i=1}^{N} \mathbf{1}(\hat{y}_i = y_i),
\)
where $\mathbf{1}(\hat{y}_i = y_i)$ represents the indicator function, which equals 1 if $\hat{y}_i = y_i$ and 0 otherwise. This metric measures the proportion of examples where the predicted label matches the ground truth.

To further assess the computational cost during training, inspired by~\cite{zhang2021self}, we fairly compare the efficiency of different methods by considering both the number of epochs required to reach specified error rates and the computational complexity per epoch. Formally, the computational cost for an algorithm is defined as:
$Cost = \frac{1}{L}\sum_{l=1}^{L}\text{Argmin}(f(x)\leq Err_l)\times \Omega_e.$
In this formula, $L$ represents the number of predefined error rate thresholds (set to 5 in our experiments, $Err_l$ denotes the predefined error rate levels, $\text{Argmin}(f(x)\leq Err_l)$ is the first epoch at which the algorithm reaches or goes below the specified error rate $Err_l$ and $\Omega_e$ represents the algorithmic complexity per epoch, measured in floating-point operations (FLOPs). 
We define the error rate $Err_l$ using upper and lower bounds determined from the training curves of all methods under comparison. Specifically, The upper bound is set as the minimum value among the highest error rates of all compared methods. The lower bound is set to the maximum of the lowest values of the final error rates of the various methods. Within this range, we choose error rate thresholds at uniform intervals to ensure that the entire performance interval has been systematically and fairly evaluated. 

(2) Audio-Visual Continual Learning. We track model accuracy throughout the training process using two metrics: average accuracy (AA) and average incremental accuracy (AIA).
At each step $k$ (i.e., after learning the $k$-th task), we compute the average accuracy $\mathrm{AA}_k$ across all tasks encountered so far as:
\(
\mathrm{AA}_k = \frac{1}{k} \sum_{j=1}^k a_{k,j},
\)
where $a_{k,j}$ is the accuracy on the $j$-th task after learning task $k$, and $j \leq k$. 
To summarize performance across the entire learning sequence, we report the Average Incremental Accuracy, which is the mean of AA values over all $K$ tasks:
\(
\mathrm{AIA} = \frac{1}{K} \sum_{i=1}^K \mathrm{AA}_i.
\)
Here, AA reflects the model’s performance after each task, while AIA captures the overall trend and stability of learning across all tasks.

To further quantify how much the model forgets previous tasks, we introduce the average forgetting rate (AFR) in the appendix Table~\ref{tab-avcil}. Let $a_{k,j}$ be the test accuracy on task $j$ after learning task $k$. Define the forgetting on task $k$ as
\(
F_k = \frac{1}{k-1} \sum_{j=1}^{k-1} \max_{1 \le \ell \le k-1} \left(a_{\ell,j} - a_{k,j}\right),
\)
i.e., the average drop from the highest accuracy ever achieved on task $j$ to its accuracy after the final task. For a total of K tasks, then
\(
\mathrm{AFR} = \frac{1}{K-1} \sum_{k=2}^{K} F_k,
\)
where $F_k$ excludes the first task (as forgetting can only be measured after learning at least two tasks). This metric summarizes how much the model’s performance on previously learned tasks degrades over the entire learning sequence.

(3) Audio-Visual Question Answering. We evaluate modal-specific question and answer accuracy, denoted by $(A_a,A_v,A_{av})$, which respectively evaluate performance across audio-only, visual-only, and audio-visual questions. Each accuracy is computed as
\(
\tfrac{1}{N} \sum_{i=1}^{N} \mathbf{1}(\widehat{\mathrm{ans}}_i = \mathrm{ans}_i),
\)
which measures exact match between predicted and ground-truth answers over all question types within each modality.

\bigskip
\noindent
\textbf{Code and reproducibility}\par
\noindent
The code implementation is based on Pytorch~\cite{NEURIPS2019_bdbca288}. The full source code, configuration files and pre‑trained checkpoints are released under an MIT licence at \texttt{https://github.com/Brain-Cog-Lab/IEMF}.

\section{Acknowledgment}
This work is supported by National Natural Science Foundation of China (NSFC) Young Scientists Fund (Grant No. 62406325).

{\small
\bibliographystyle{ieee_fullname}
\bibliography{egbib}
}

\onecolumn
\appendix
\section{appendix}
\renewcommand{\thetable}{S\arabic{table}}
\renewcommand{\thefigure}{S\arabic{figure}}
\renewcommand{\thealgorithm}{S\arabic{algorithm}}
\renewcommand{\theequation}{S\arabic{equation}}
\setcounter{table}{0}
\setcounter{figure}{0}

\begin{table}[htb]
  \centering
  \resizebox{1.0\linewidth}{!}{
  \begin{tabular}{|l|l|l|}
    \hline
    \textbf{Category} & \textbf{Parameters} & \textbf{Values} \\
    \hline
    \multirow{6}{*}{Audio Visual Classification} 
    & Network backbone & ResNet-18 \\
    \cline{2-3}
    & Optimizer & SGD \\
    \cline{2-3}
    & Weight decay & $1\times10^{-4}$ \\
    \cline{2-3}
    & Initial learning rate & $5\times10^{-3}$ \\
    \cline{2-3}
    & Number of training epochs & 100 \\
    \cline{2-3}
    & Batch size & 32 \\
    \hline
    \multirow{4}{*}{Audio Visual Continual Learning} 
    & Learning rate & $1\times10^{-2}$ \\
    \cline{2-3}
    & Batch size & 256 \\
    \cline{2-3}
    & Number of training epochs & 100 \\
    \cline{2-3}
    & Optimizer (VS100-CI) & Adam ($\beta_{1}=0.9$, $\beta_{2}=0.999$) \\
    \hline
    \multirow{3}{*}{Audio Visual Question and Answering} 
    & Learning rate & $1\times10^{-2}$ \\
    \cline{2-3}
    & Batch size & 64 \\
    \cline{2-3}
    & Number of training epochs & 50 \\
    \hline
    \multirow{7}{*}{LIF Neuron (SNN)} 
    & Resting potential $V_{\text{rest}}$ & 0 \\
    \cline{2-3}
    & Firing threshold $V_{\text{th}}$ & 0.5 \\
    \cline{2-3}
    & Membrane time constant $\tau_m$ & 2.0 \\
    \cline{2-3}
    & Surrogate gradient function & Piecewise linear~\cite{bellec2018long} \\
    \cline{2-3}
    & Conductivity $g_L, g_E, g_I$ & $1\,\mathrm{S}, 1\,\mathrm{S}, 1\,\mathrm{S}$ \\
    \cline{2-3}
    & Reversal potential $V_E, V_I$ & 0\\
    \cline{2-3}
    & Discrete time step $t$ & $4$ \\
    \hline
    \multirow{2}{*}{IEMF} 
    & Inverse gain coefficient $\gamma$ & 1.0 \\
    \cline{2-3}
    & Gating function $\kappa(\cdot)$ & $\tanh(\cdot)$ \\
    \hline
  \end{tabular}
  }
  \caption{ 
    Comprehensive experimental parameters used for implementation across three multimodal tasks.}
  \label{tab:snn_ann_parameters}
\end{table}

\begin{table}[htb]
  \centering
  \begin{threeparttable}
    \resizebox{1.0\linewidth}{!}{
      \begin{tabular}{ccccccccccccc}
        \toprule
        \multirow{2}{*}{Methods} & \multicolumn{4}{c}{CREMA-D} &\multicolumn{4}{c}{Kinetics-Sounds} & \multicolumn{4}{c}{UrbanSound8K-AV} \\
        \cmidrule(lr){2-5} \cmidrule(lr){6-9} \cmidrule(lr){10-13}
        & Normal & MSLR & OGM\_GE & LFM & Normal & MSLR & OGM\_GE & LFM & Normal & MSLR & OGM\_GE & LFM\\
        \midrule
        Vanilla & $ 62.63 $ & $64.11$ & $ 68.68$ & $\textbf{64.11}$ & $51.58$ & $51.89$ & $57.63$ & 55.28 & $\textbf{97.90}$ & 97.79 & 97.60 & 98.05\\
        w/ IEMF & $ \textbf{63.44} $ & $ \textbf{65.59} $ & $\textbf{71.10}$ &63.98 & $\textbf{56.17}$ & $\textbf{55.86}$& $\textbf{64.61}$ & $\textbf{63.15}$ & 97.86 & $\textbf{97.98}$ & $\textbf{99.24}$ & $\textbf{98.63}$\\
        \bottomrule
    \end{tabular}
    }
  \end{threeparttable}
  \caption{  
  Comparison of the proposed method (w/ IEMF) and the vanilla baseline on ANN across three multimodal datasets—CREMA-D, Kinetics-Sounds, and UrbanSound8K-AV—under four fusion methods (Normal, MSLR, OGM\_GE, and LFM). Bold values indicate the highest accuracy achieved in each configuration.}
  \label{acc-ann}
\end{table}

\begin{table}[htb]
  \centering
  \begin{threeparttable}
    \resizebox{1.0\linewidth}{!}{
      \begin{tabular}{ccccccccccccc}
        \toprule
        \multirow{2}{*}{Methods} & \multicolumn{4}{c}{CREMA-D} &\multicolumn{4}{c}{Kinetics-Sounds} & \multicolumn{4}{c}{UrbanSound8K-AV} \\
        \cmidrule(lr){2-5} \cmidrule(lr){6-9} \cmidrule(lr){10-13}
        & Normal & MSLR & OGM\_GE & LFM & Normal & MSLR & OGM\_GE & LFM & Normal & MSLR & OGM\_GE & LFM\\
        \midrule
        Vanilla & $ 63.04 $ & 63.70 & $ 69.49 $ & 63.58 & $53.12$ & 53.19 & $57.28$ & 54.63 & $\textbf{98.02}$ & 97.98 & 97.67 & 98.13\\
        w/ IEMF & $ \textbf{64.65} $ & $\textbf{64.38}$ & $\textbf{69.35}$ & $\textbf{64.78}$ & $\textbf{54.97}$ & $\textbf{56.63}$ & $\textbf{62.88}$ & $\textbf{63.53}$ & 97.79 & $\textbf{98.05}$ & $\textbf{99.35}$ & $\textbf{98.63}$\\
        \bottomrule
    \end{tabular}
    }
  \end{threeparttable}
  \caption{  
  Same as Table~\ref{acc-ann}, with the difference that evaluation is conducted on SNN.}
  \label{acc_snn}
\end{table}

\begin{table}[htb]
  \centering
  \begin{threeparttable}
    \resizebox{0.95\linewidth}{!}{
      \begin{tabular}{ccccccc}
        \toprule
        \multirow{2}{*}{Method} & \multicolumn{2}{c}{AVE-CI}  & \multicolumn{2}{c}{K-S-CI } & \multicolumn{2}{c}{VS100-CI} \\
        & {Mean Accuracy $\uparrow$} & {Avg Forgetting $\downarrow$} &{Mean Accuracy $\uparrow$} & {Avg Forgetting $\downarrow$} &{Mean Accuracy $\uparrow$} & {Avg Forgetting $\downarrow$}\\
        \cmidrule(lr){1-1} \cmidrule(lr){2-3} \cmidrule(lr){4-5} \cmidrule(lr){6-7} 
        LwF~\cite{li2017learning} & 54.06 & 26.77 & 59.89 & \textbf{15.26} & 59.06 & \textbf{17.91}\\
        LwF w/ IEMF & \textbf{56.00} & \textbf{24.37} & \textbf{62.79} & 17.25 & \textbf{60.72} & 18.25\\
        \midrule
        SSIL~\cite{ahn2021ss} & 56.19 & 7.61 & 63.31 & \textbf{4.66} & 68.21 & \textbf{8.53}\\
        SSIL w/ IEMF & \textbf{58.12} & \textbf{6.25} & \textbf{65.18} & 5.45& \textbf{70.01} & 9.30\\
        \midrule
        AV-CIL~\cite{pian2023audio} & 61.86 & 24.31 & 70.63 & 11.03 & 70.15 & \textbf{8.98}\\
          AV-CIL w/ IEMF & \textbf{61.94} & \textbf{20.87} & \textbf{72.49} & \textbf{10.44} & \textbf{70.76} & 9.49\\
        \bottomrule
    \end{tabular}
    }
  \end{threeparttable}
  \caption{ 
    Performance comparison on three audio-visual continual learning benchmarks (AVE-CI, K-S-CI, and VS100-CI) in terms of Mean Accuracy ($\uparrow$) and Average Forgetting ($\downarrow$). We evaluate three baseline methods (LwF, SSIL, and AV-CIL) and their variants augmented with the proposed IEMF module. Bold values indicate the best performance in each metric.}
  \label{tab-avcil}
\end{table}

\begin{table}[htb]
  \centering
  \begin{threeparttable}
  \resizebox{1.0\linewidth}{!}{ 
  \begin{tabular}{l ccc ccc cccccc c}
  \toprule
  \multirow{2}{*}{\textbf{Method}} 
  & \multicolumn{3}{c}{\textbf{Audio Question (\%)}} 
  & \multicolumn{3}{c}{\textbf{Visual Question (\%)}} 
  & \multicolumn{6}{c}{\textbf{Audio-Visual Question (\%)}} 
  & \multirow{2}{*}{\textbf{Overall Avg. (\%)}} 
  \\ 
  \cmidrule(lr){2-4} \cmidrule(lr){5-7} \cmidrule(lr){8-13}
   & \textbf{Counting} & \textbf{Comparative} & \textbf{Avg} 
   & \textbf{Counting} & \textbf{Location} & \textbf{Avg} 
   & \textbf{Existential} & \textbf{Location} & \textbf{Counting} & \textbf{Comparative} & \textbf{Temporal} & \textbf{Avg} 
   & 
   \\
  \midrule
 Baseline
      & 77.20 & 62.06 & 71.60
      & 74.15 & \textbf{76.79} & 75.48
      & 81.71 & 67.43 & \textbf{62.57} & 61.61 & 62.99 & 67.34
      & 70.24
      \\
  Baseline w/ IEMF
      & \textbf{77.40} & \textbf{63.89} & \textbf{72.40}
      & \textbf{75.23} & \textbf{76.79} & \textbf{76.02}
      & \textbf{82.11} & \textbf{69.07} & 59.55 & \textbf{62.87} & \textbf{64.93} & \textbf{67.87}
      & \textbf{70.82}
      \\
  \midrule
  ST-AVQA~\cite{li2022learning} 
  & 77.59 & 62.23 & 71.90
  & 73.89 & 75.57 & 74.74
  & \textbf{82.81} & 68.45 & \textbf{63.00} & 60.45 & 62.86 & 67.61
  & 70.26
  \\
ST-AVQA w/ IEMF
  & \textbf{79.84}  & \textbf{65.39}  & \textbf{74.49} 
  & \textbf{74.65}  & \textbf{76.63}  & \textbf{75.65} 
  & 82.11  & \textbf{69.47}  & 62.68  & \textbf{62.51}  & \textbf{64.20}  & \textbf{68.33} 
  & \textbf{71.36} 
  \\
  \bottomrule
  \end{tabular}
  }
  \end{threeparttable}
  \caption{
    Comparison results with different AVQA methods on the MUSIC-AVQA dataset, where different types of questions (audio-only, visual-only, and audio-visual) are evaluated.
  }
  \label{tab:music_avqa_comparison}
\end{table}

\end{document}